\documentclass[11pt]{article}
\usepackage[margin=1in]{geometry} 
\usepackage{graphicx}
\usepackage{url}
\usepackage{authblk}

\usepackage{amsmath,amssymb, algorithm,algpseudocode}
\usepackage{microtype}
\usepackage{graphicx}
\usepackage{subcaption}
\usepackage{booktabs} 
\usepackage{float}
\usepackage{hyperref}
\usepackage{mathtools}
\usepackage{amsthm}
\usepackage[capitalize,noabbrev]{cleveref}

\theoremstyle{plain}
\newtheorem{theorem}{Theorem}[section]

\newtheorem{lemma}[theorem]{Lemma}

\theoremstyle{definition}
\newtheorem{definition}[theorem]{Definition}
\newtheorem{assumption}[theorem]{Assumption}
\theoremstyle{remark}

\newcommand{\obj}{\ensuremath{f(x)}}
\newcommand{\gradtrue}{\ensuremath{g(x)}}
\newcommand{\gradest}{\ensuremath{\tilde{g}(x)}}
\newcommand{\sign}[1]{\textnormal{sign}\left(#1\right)}
\newcommand{\signgradtrue}{\ensuremath{\textnormal{sign}\left(g(x)\right)}}
\newcommand{\signgradest}{\ensuremath{\textnormal{sign}\left(\tilde{g}(x)\right)}}
\newcommand{\Prob}{\mathbb{P}}
\newcommand{\E}{\mathbb{E}}
\newcommand{\Var}{\operatorname{Var}}
\DeclareMathOperator*{\argmax}{argmax}

\newcommand{\calB}{\mathcal B}

\title{On the Byzantine Fault Tolerance of signSGD with Majority Vote}

\author[1]{Emanuele Mengoli\thanks{Corresponding author: \texttt{emanuele.mengoli@polytechnique.edu}}}
\author[2]{Luzius Moll\thanks{Corresponding author: \texttt{luzius.moll@tum.de}}}
\author[3]{Virgilio Strozzi\thanks{Corresponding author: \texttt{vstrozzi@student.ethz.ch}}}
\author[1]{El-Mahdi El-Mhamdi}

\affil[1]{École Polytechnique}
\affil[2]{TU Munich}
\affil[3]{ETH Zurich}

\begin{document}

\maketitle

\begin{abstract}
   In distributed learning, sign-based compression algorithms such as signSGD with majority vote provide a lightweight alternative to SGD with an additional advantage: fault tolerance (almost) for free. However, for signSGD with majority vote, this fault tolerance has been shown to cover only the case of weaker adversaries, i.e., ones that are not \emph{omniscient} or cannot \emph{collude} to base their attack on common knowledge and strategy. In this work, we close this gap and provide new insights into how signSGD with majority vote can be resilient against \emph{omniscient} and \emph{colluding} adversaries, which craft an attack after communicating with other adversaries, thus having better information to perform the most damaging attack based on a common optimal strategy. Our core contribution is in providing a proof that begins by defining the omniscience framework and the strongest possible damage against signSGD with majority vote without imposing any restrictions on the attacker. 
   Thanks to the filtering effect of the sign-based method, we upper-bound the space of attacks to the optimal strategy for maximizing damage by an attacker.
   Hence, we derive an explicit probabilistic bound in terms of incorrect aggregation without resorting to unknown constants, providing a convergence bound on signSGD with majority vote in the presence of Byzantine attackers, along with a precise convergence rate. Our findings are supported by experiments on the MNIST dataset in a distributed learning environment with adversaries of varying strength.
\end{abstract}

\section{Introduction}

As AI models grow in size, training them motivates multiple forms of distributed learning, ranging from the fully centralized parameter-server setting~\cite{li2014scaling} to fully decentralized setups~\cite{jungle}. Common to all forms of distributed learning, several challenges have been the subject of intense research over the past decade. A non-exhaustive list of these challenges includes (1) robustness of the distributed learning system to adversarial behavior by some of its members, (2) communication costs among participants in the learning procedure, and (3) preserving the privacy of users' data while learning in a distributed fashion.

Focusing on Stochastic Gradient Descent (SGD), the workhorse of today's most successful AI applications, an important body of work has addressed the first~\cite{blanchard2017machine, el2020robust, lilisu, jungle, alistarh2018byzantine, tournesol, boussetta2021aksel, rouault2022practical}, the second~\cite{qsgd, gandikota2021vqsgd, karimireddy2019error, haddadpour2021federated}, and the third~\cite{ji2014differential, li2020privacy, kairouz2021distributed, fang2021privacy, sun2018private} challenges separately. In this work, we are interested in understanding how well the first and second challenges can be tackled simultaneously. Namely, we ask whether it is possible to reduce the communication cost (2) of distributed learning while simultaneously remaining robust to adversarial behavior within the system (1).

A first partially positive answer to our question was provided in~\cite{bernstein2018signsgd} using signSGD\footnote{For brevity, we refer to "signSGD with majority vote" simply as "signSGD" throughout the rest of this paper.}, a variant of SGD in which only the signs of each component of the minibatch stochastic gradient are transmitted. Specifically, it was shown that the signSGD algorithm is fault tolerant, but only against a weak class of attackers — namely, blind multiplicative adversaries — that cannot communicate or coordinate to carry out a shared (optimal) attack.

The core of our contribution is proving the resilience of the signSGD algorithm against \emph{omniscient} and \emph{colluding} adversaries — entities capable of orchestrating more sophisticated attacks by pooling resources and information with other adversaries. Thanks to the nature of signSGD, we can narrow down the damage that can be caused by \emph{any possible attack on signSGD}, bounding it by the damage of a specific \emph{optimal attack}.

In the parlance of distributed computing~\cite{lamport}, we consider not only the fault tolerance of signSGD but also its \emph{Byzantine Fault Tolerance} (BFT). In the realm of distributed machine learning, the BFT property refers to the algorithm's resilience against omniscient malicious workers attempting to prevent convergence in a distributed setting~\cite{blanchard2017machine, alistarh2018byzantine, mhamdi2018hidden, bernstein2018signsgd}. This represents an extension of the original definition in distributed systems, where BFT refers to the system's ability to continue functioning (and reach consensus) despite the presence of malicious, omnipotent attackers~\cite{lamport, castro1999practical}.

Our objective is to build upon the findings of~\cite{bernstein2018signsgd} by addressing and defining a more powerful class of attackers, namely Byzantine adversaries, who are capable of launching arbitrary attacks — in contrast to the weaker adversaries considered in the original work of~\cite{bernstein2018signsgd}. It is worth noting that this gap was explicitly acknowledged in~\cite{bernstein2018signsgd}, where fault tolerance against blind multiplicative adversaries was demonstrated for signSGD, but BFT was not established and the gap never closed. Furthermore, in light of a concurrent and recent proof presented in~\cite{jin2024sign}, which also argues for the BFT of the signSGD algorithm, we propose a comparative analysis of both proofs. Our analysis focuses on evaluating the theoretical limitations and proof methodologies of the approach proposed by~\cite{jin2024sign} compared to ours. In particular, we do not resort to any unknown constants, and we consider truly Byzantine attackers in the context of the signSGD framework — i.e., adversaries that can be omniscient, know the exact value of the true gradient, and are not captured by probabilities on their behavior.

The structure of the paper is organized as follows: Section~\ref{sec:model} details our model, notation, and useful definitions; Section~\ref{sec:theoretical-contribution} provides our analysis and establishes the BFT of signSGD; Section~\ref{sec:comparative} offers a comparative analysis with~\cite{jin2024sign}, which is, to the best of our knowledge, the only other work arguing for the BFT of signSGD; Section~\ref{sec:experiments} experimentally supports our findings; and finally, Section~\ref{sec:conclusion} concludes the paper and discusses its limitations.

\section{Model}
\label{sec:model}

\begin{algorithm}[ht!]
    \caption{\textsc{signSGD} with majority vote (Proposed in \cite{bernstein2018signsgd}).}
    \label{singSGD algo}
\begin{algorithmic}
\Require learning rate $\eta > 0$, weight decay $\lambda \geq 0$, mini-batch size $n$, initial point $x_t$ at each of $M$ workers.
\Repeat
    \State \textbf{on} $m^{\text{th}}$ worker
    \State \hskip2em $\tilde{g}_{m,t} \leftarrow \frac{1}{n} \sum_{i=1}^n 
    g_{m,i}(x_t)$
    \State \hskip2em \textbf{push} $\text{sign}(\tilde{g}_{m,t})$ \textbf{to} server 
    \State \textbf{on} server
    \State \hskip2em $O_t \leftarrow \sum_{m=1}^M \text{sign}(\tilde{g}_{m,t})$  
    \State \hskip2em \textbf{push} $\text{sign}(O_t)$ \textbf{to} each worker 
    \State \textbf{on} every worker
    \State \hskip2em $x_{t+1} \leftarrow x_t - \eta (\text{sign}(O_t) + \lambda x_t)$
\Until convergence
\end{algorithmic}
\end{algorithm}

We consider the now standard setting of the signSGD with majority vote algorithm, as proposed in the seminal work of \cite{bernstein2018signsgd}, shown in Algorithm \ref{singSGD algo}. In this context, workers communicate solely the sign for each component of their gradient to the server, thereby improving communication efficiency and robustness against malicious actors. 

The key notation used throughout this work is presented in Table~\ref{tab:notation}, providing clarity on the variables and concepts employed in the analysis.

\begin{table}[ht!]
\caption{Key Notation}
\label{tab:notation}
\vskip 0.15in
\begin{center}
\begin{small}
\begin{tabular}{p{2cm} p{13cm}}
\toprule
Notation & Description  \\
\midrule
$O$ & Output gradient of the server after the majority vote. \\ 
\obj & Objective function, e.g., Loss function. \\
\gradtrue & \emph{True} gradient of the objective function \obj, with respect to the parameters $\theta$. \\ 
\gradest & \emph{Estimate} gradient of the objective function \obj, before taking the majority vote $O = \text{sign} (\gradest)$. \\ 
$\gradtrue_{i}$ & \emph{True} gradient of the objective function at index $i$. \\ 
$\gradest_{i, t, v}$ & \emph{Estimate} gradient at index $i$, at time step $t$ for a worker $v$. \\ 
$M, \calB$ & The honest and adversary workers, respectively. Total number of worker is $Q=M+ \calB$. \\ 
$K$, $K'$ & The sum of contributions of honest $M$ and adversary $\calB$ workers, respectively. It holds $O = \signgradest = \text{sign}(K + K')$. \\ 
\bottomrule
\end{tabular}
\end{small}
\end{center}
\vskip -0.1in
\end{table}

The Byzantine framework we borrow from the large literature on BFT learning~\cite{lamport, blanchard2017machine, lilisu} and want to emphasize that \emph{worst case} and \emph{Byzantine} attackers are identical situations.

We are going back to the formal definition of Byzantine Faults \cite{lamport} — characterized by omniscient, arbitrarily colluding attackers — and demonstrating how signSGD can withstand such threats. This distinction is essential, as the strategies employed by malicious workers to undermine convergence differ significantly under the Byzantine Fault model. 

\begin{definition} \label{omniscient} \textnormal{(Omniscient Adversaries).} A system consists of $Q$ workers, where a subset of $\alpha Q$ are 	\emph{omniscient adversaries}. These adversaries aim to prevent the convergence of 	\emph{signSGD with majority vote} by leveraging all capabilities allowed for a Byzantine adversary. In particular, 
\begin{itemize}
    \item they coordinate their votes by accessing all $\alpha Q \in \calB$ adversarial gradients $\gradest_{t, v}$ for every adversary $v \in \calB$ before transmitting $\signgradest_{t, v}$ to the server, 
    
    \item they possess knowledge of the true gradient $\gradtrue_t$ at each step $t$, 
    \item they observe the gradients $\gradest_{t, w}$ submitted by all honest workers $w \in M$ at step $t$ and
    \item they know the current and past values of the objective function $\obj_t$, but not its~future~values.
\end{itemize}

\end{definition}

\section{Analysis}
\label{sec:theoretical-contribution}

This section aims to demonstrate that \textbf{signSGD can still converge in a non-convex scenario} to a critical point, even when up to $\alpha < 1 - \frac{1}{2p}$ of the workers are \textbf{omniscient and colluding adversaries}. Here, $p$ denotes the probability of a worker correctly guessing the sign of the true gradient at an arbitrary index $i$.

\subsection{Assumptions}

For our analysis, we take the following assumptions as in the original work of \cite{bernstein2018signsgd}.

\begin{assumption} \textnormal{ (Lower Bound).}
    For all $x$ and some constant $f^*$, we have objective value $\obj \geq f^*$. This assumption is standard and necessary for guaranteed convergence to a stationary point.
\end{assumption}

The next three assumptions naturally encode notions of heterogeneous curvature and gradient noise.

\begin{assumption} \textnormal{(Smooth).}
    Let $g(x)$ denote the gradient of the objective function $f(\cdot)$ evaluated at point $x$. Then $\forall x, y$ we require that for a non-negative constant $L$,
    \begin{equation*}
    f(y)  - \obj  \leq  \gradtrue^T(y-x) + \frac{L}{2} \|(y - x)\|^2.
    \end{equation*}
\end{assumption}

\begin{assumption} \textnormal{(Variance Bound).}
    Upon receiving query $x \in \mathbb{R}^d$, the stochastic gradient oracle gives us an independent unbiased estimate \gradest{} that has coordinate bounded variance:
    \begin{equation*}
    \E\left[\tilde{g}(x)\right] = g(x), \quad \E\left[(\gradest - \gradtrue)^2\right] \leq \sigma^2
    \end{equation*}
for a of non-negative constant $\sigma$.
\end{assumption}

\begin{assumption} \textnormal{(Unimodal, Symmetric Gradient Noise).}
    At any given point x, each component of the stochastic gradient vector \gradest{} has a unimodal distribution that is also symmetric about the mean.
\end{assumption}

\subsection{Narrowing Down the Damage of an Arbitrary Attack}

Due to the filtering effect of sign-based methods, the damage caused by an arbitrary attack is upper-bounded by flipping the the majority vote outcome for each component.
This is the core difference compared to magnitude-based approache, where the adversaries can performe more fine-grained gradient manipulations. 
Thus, the maximum damage adversaries can inflict on the convergence of signSGD is to act on the objective function $\obj$ by consistently attempting to flip the majority vote of $\gradest$, thereby pushing its outcome as far as possible from the true optimization direction. This is achieved by flipping every possible component of the gradient, ensuring that the estimated gradient direction is maximally misaligned with the true gradient direction at each step.

\begin{theorem} \textnormal{(Strongest Damage\footnote{Note that this is about bounding the damage of an arbitrary attack, and not assuming a particular attack.} on signSGD with Majority Vote).} \label{theorem_strongest_attack}
The strongest attack to maximally damage the objective function $\obj$ at time $t$ that omniscient adversaries (controlling $\alpha Q$ workers) can execute is to transmit:
\begin{equation*} 
K'_t = \sum\nolimits_{v \in \calB} -\signgradtrue_t = \alpha Q (-\signgradtrue_t),
\end{equation*}
forcing the majority vote toward the opposite of $\gradtrue_t$.
\end{theorem}

\begin{proof}
Due to the sign-based nature of signSGD, adversaries cannot manipulate gradient magnitudes but can only influence the direction of the majority vote. This restricts attack strategies to merely altering the number of incorrect signs in $\gradest_t$. Since signSGD acts as a "filter," removing magnitude information, and adversaries have access only to the gradient estimate $\gradest_{t, w}$ of the honest workers and the true gradient $\gradtrue_t$ at the current time $t$, without the ability to brute-force all possible future values of $f(x)$ to devise a more sophisticated attack, the optimal strategy for maximally disrupting the convergence of signSGD is purely directional: flipping as many signs as possible to push away the the outcome of the objective function $\obj$ from its closest convergent value.
The \emph{strongest attack} maximizes the deviation from the true gradient direction by ensuring: \begin{itemize} \item $\forall \gradtrue_{i,t}, \quad \sign{K'_{i,t}} \neq \sign{\gradtrue_{i,t}}$.
\item $\forall \gradtrue_{i,t}, \quad K'_{i,t} = \argmax{K'_{i,t}} |\gradtrue_{i,t} - K'_{i,t}|$.
\end{itemize}
Thus, adversaries optimally shift $\gradest_i$ in the opposite direction of $\gradtrue_i$ before the majority vote finalizes the sign.
Applying this to all $d$ dimensions maximizes the number of flipped gradient directions, making it the strongest possible attack within the constraints of sign-based SGD.
\end{proof}

\subsection{Convergence of signSGD in Presence of Omniscient Adversaries}
First, we need to bound the probability of a worker guessing the wrong sign. The following proof closely follows the approach in \cite{bernstein2018signsgd}, specifically their Lemma 1.

\begin{lemma}
\textnormal{(Accuracy of Sign Guessing).}
\label{Lemma 1}
Let $\gradest_i$ be an unbiased stochastic approximation to the gradient component $\gradtrue_i$, with variance bounded by $\sigma^2_i$. Assume the noise distribution is unimodal and symmetric. Define the signal-to-noise ratio (SNR) as $S_i := \frac{|\gradtrue_i|}{\sigma_i} > 0$. Then, we have that:
\begin{equation*}
    1 - p = \Prob[\sign{\gradest_i} \neq \sign{\gradtrue_i}] \leq \frac{1}{2} - \frac{S_i}{2\sqrt{4 + S_i^2}}.
\end{equation*}
which in all cases is less than or equal to $\frac{1}{2}$. Thus,

\begin{equation*}
    \frac{p(1 - p)}{(p - \frac{1}{2})^2} \leq \frac{4}{S_i^2}
\end{equation*}
\end{lemma}

\paragraph{Proof Sketch and Remarks.} For the first equality, the proof consists in a direct straightforward computation, obtained by comparing the inequalities with the bound proposed in Lemma 1 in \cite{bernstein2018signsgd}. As for the second inequality, a direct straightforward computation can be found in appendix \ref{prooflemma} along with the full proof of Lemma \ref{Lemma 1}. In addition to the proofs (deferred to the appendix), we can add the two following remarks.

\paragraph{Monotonicity.} The probability of guessing the wrong sign monotonically decreases from $\frac{1}{2}$ to 0 as the SNR increases. This behavior makes sense, as a higher gradient compared to the noise provides more information.

\paragraph{Critical Points.} When the norm of the gradient approaches 0, the SNR tends to 0, and the bound results in a probability of guessing the wrong sign of $\frac{1}{2}$, as in random guessing. However, it has been argued in the past that when the gradient becomes such, BFT stops \cite{mhamdi2018hidden, baruch2019little}, since a Byzantine worker can exploit small gradient norms to hinder learning\footnote{more precisely, small ratio between the norm of the gradients from correct workers and the variance of the correct workers' gradients increase the likelihood of successful attacks that exploit legitimate disagreements between honest workers~\cite{momentum}}. Hence, it is necessary to introduce conditions on the accuracy of gradient estimator, which is why momentum techniques \cite{momentum}) are often introduced to improve the convergence speed of gradient-based methods.

\begin{theorem} \textnormal{(Non-convex Convergence Rate of signSGD with Majority Vote Against Omniscient Adversaries).}
Execute Algorithm \ref{singSGD algo} for $K$ iterations under Assumptions 1 to 4. Disable weight decay ($\lambda = 0$). Set the learning rate $\eta$, and mini-batch size $n$ for each worker as
\begin{equation*}
    \eta = \sqrt{\frac{f_0 - f_*}{\|L\|_{1}K}}, \quad n = K.
\end{equation*}
Assuming that a fraction $\alpha < 1 - \frac{1}{2p}$, where $p = \Prob[\text{sign}(\gradest_i) = \text{sign}(\gradtrue_i)]$ (i.e., the probability of a worker correctly guessing the sign of $\gradtrue_i$ under the unimodal, symmetric gradient noise assumption), of the $Q$ workers behaves adversarially as defined by omniscient adversaries, then the majority vote converges at rate:
\begin{equation*}
  \left[\frac{1}{K}\sum^{K - 1}_{k = 0} E\|g_k\|_1\right]^2 \leq \frac{4}{\sqrt{N}}  \left[\frac{\|\sigma\|_1}{4\sqrt{Q}}\frac{\sqrt{(1- \alpha)p}}{((1 - \alpha)p - \frac{1}{2})}  + \sqrt{\|L\|_1(f_0 - f_*)} \right] ^2
\end{equation*}
where $N=K^2$ is the total number of stochastic gradient calls per worker up to step $K$.
\end{theorem}

\begin{proof}
\noindent We aim to bound the failure probability of the vote in the worst-case scenario, when adversaries behave the damage, and use this bound to derive a convergence rate. As demonstrated in \emph{Theorem \ref{theorem_strongest_attack}}, the strongest attack omniscient adversaries can craft is to \emph{send the opposite sign value of the real gradient for each entry $i$}. Thus, we limit our analysis to this scenario to find an upper bound for the convergence rate of signSGD with majority vote.

Consider $(1 - \alpha)Q$ honest machines and $\alpha Q$ adversaries. The honest workers compute a stochastic gradient estimate, evaluate its sign, and transmit it to the server. The adversaries send the opposite sign of the true gradient $g(x)$. Depending on the adversaries' proportion $\alpha$ and on the probability $p$ of a honest workers of correctly guessing the sign, the honest workers will, on average, win the vote. Our goal is to determine how many adversarial workers, $\alpha Q$, can be tolerated, depending on $p$, while still allowing signSGD with majority vote to converge to a critical point.

We now bound the probability of failure for a gradient estimate $\gradest_i$. For a given gradient component $\gradtrue_i$, let the random variable $Z \in [0, Q]$ denote the number of correct sign bits received by the parameter server. Let $G$ and $B$ be the random variables denoting the number of honest and adversarial workers, respectively, who sent the correct sign bit. In our scenario, $B = 0$, as adversaries always send the wrong sign when executing their \emph{strongest attack}, thus never contributing to $Z$. We can now express $Z$ as follows:
\begin{align*}
    &Z = G + B = G, \\
    &G \sim \text{Binomial}\left[(1 - \alpha)Q, p\right], \\
    &\E[Z] = (1 - \alpha)Qp, \\
    &\Var[Z] = (1 - \alpha)Q p(1 - p).
\end{align*}
 The vote fails on index $i$ if $Z < \frac{Q}{2}$, which occurs with probability:
\begin{align*}
    \Prob[\text{vote fails for }  i^\text{th} \text{ coordinate}] %\\ 
    &=\Prob\left[Z < \frac{Q}{2}\right] \\
    &= \Prob\left[\E[Z] - Z \geq \E[Z] - \frac{Q}{2}\right] \\
    &\leq \frac{1}{1 + \frac{(\E[Z] - \frac{Q}{2})^2}{\Var[Z]}} \quad \text{(Cantelli's inequality)} \\
    &\leq \frac{1}{2}\sqrt{\frac{\Var[Z]}{(\E[Z] - \frac{Q}{2})^2}} \quad \text{(since } 1 + x^2 \geq 2x \text{)} \\
    &= \frac{1}{2\sqrt{Q}}\frac{\sqrt{(1 - \alpha)p(1- p)}}{((1 - \alpha)p - \frac{1}{2})} \\
    &= \frac{\sqrt{(1- p)}}{2\sqrt{Q}}\frac{\sqrt{(1 - \alpha)p}}{((1 - \alpha)p - \frac{1}{2})} \\
    &= \frac{\sqrt{(1- p)p}}{2\sqrt{Q}}\frac{\sqrt{(1 - \alpha)}}{((1 - \alpha)p - \frac{1}{2})}\sqrt{\frac{(p - \frac{1}{2})^2}{(p - \frac{1}{2})^2}}\\
    &\leq \frac{1}{4S_i\sqrt{Q}}\frac{\sqrt{(1 - \alpha)}(p - \frac{1}{2})}{((1 - \alpha)p - \frac{1}{2})} \quad \text{(Lemma \ref{Lemma 1}).}
\end{align*}

From the above result, we infer that for the probability to be non-negative we need $\alpha < 1 - \frac{1}{2p}$ and $p > \frac{1}{2}$.

The second stage of the proof involves deriving the convergence rate by substituting this bound into the convergence analysis of signSGD as presented by \cite{bernstein2018signsgd}. For brevity, we omit these details (refer to the original paper) and directly conclude:
\begin{equation*}
\begin{split}
    \left[\frac{1}{K}\sum^{K - 1}_{k = 0} \E\|g_k\|_1\right]^2 \leq \frac{4}{\sqrt{N}} \left[\frac{\|\sigma\|_1}{4\sqrt{Q}}\frac{\sqrt{(1 - \alpha)}(p - \frac{1}{2})}{((1 - \alpha)p - \frac{1}{2})} + \sqrt{\|L\|_1(f_0 - f_*)}\right]^2.
\end{split}
\end{equation*}
\end{proof}

The following remarks can be made from the analysis of the bound, with $\alpha \leq 1 - \frac{1}{2p}$ and $p > \frac{1}{2}$.

\paragraph{Linear Relationship Between $\alpha$ and $p$:} The allowed percentage of adversaries, $\alpha$, that still permits convergence to a critical point, increases linearly with the probability $p$ of the honest workers correctly identifying the gradient sign at a given index $i$. This relationship implies that the more accurate the honest workers are ("better" in terms of predicting gradient signs), the higher the tolerance for the presence of adversaries. Specifically, taking the extreme case where $p = 1$ leads to $\alpha < 1 - \frac{1}{2p} = \frac{1}{2}$, indicating that up to, but not including, half of the total workers can be adversarial while still achieving convergence on average.

\paragraph{Implications of the Bound on $\alpha$ and $p$:} The bounds $\alpha \leq 1 - \frac{1}{2p}$ and $p > \frac{1}{2}$ inherently imply that $\alpha < \frac{1}{2}$ and $p > \frac{1}{2}$ are necessary conditions for convergence. This aligns intuitively with the design of Majority Vote mechanisms, which are robust against up to half of the participants being adversarial in signSGD contexts. Moreover, if $p \leq \frac{1}{2}$, this suggests that random guessing or even inverting the sign might yield better results. This highlights the importance of having a majority of honest workers who are more likely than not to correctly predict the sign of the true gradient.

\paragraph{Analysis of the Fraction Involving $p$ and $\alpha$:} Focusing on the fraction $\frac{\sqrt{(1- \alpha)}(p - \frac{1}{2})}{(1 - \alpha)p - \frac{1}{2}}$, it is observed that while the numerator always exceeds the denominator when considering $\alpha \leq 1 - \frac{1}{2p}$ and $p > \frac{1}{2}$, its rate of increase is slower for all growing values of $p$ within the interval $[\frac{1}{2}, 1]$ or for diminishing values of $\alpha$. This behavior suggests that the bound not only provides a meaningful restriction but also reflects the expectation of diminishing failure rates as the probability $p$ of correct gradient sign prediction increases or the percentage of adversaries diminish.

These considerations underscore the nuanced relationship between the accuracy of honest workers' gradient sign predictions and the system's resilience to adversarial presence, elucidating the theoretical foundation for ensuring convergence in signSGD under adversarial conditions.

\section{Comparative Analysis}
\label{sec:comparative}

In \cite{jin2024sign}, the authors claimed to have proven the BFT of generalized sign-based compressors under both heterogeneous and homogeneous conditions. The main points of their study can be summarized as follows:

\begin{itemize}
    \item Threat Model: non-traditional Byzantine framework. 
    \item Incorrect aggregation considering both $\calB$ adversaries and $M$ honest workers, where $Q = \calB + M$.
\end{itemize}

In contrast to our scenario, their threat model framework measures the strength of attackers based on the probability that they will send the gradient’s opposite sign, this is inconsistent with the fact that Byzantine attackers are by definition arbitrary, omniscient and are allowed to know what prediction we make about their behaviour. Their framework lacks a formal definition of the strongest damage and does not address cooperation among attackers. In addiction, attackers use the same procedure as regular workers to estimate the gradient. The concept of omniscience is not explicitly defined, yet the framework considers a heterogeneous environment where each honest worker has a probability \( p_i \) of identifying the correct gradient $ \forall i \in M$.

Theorem 4 of \cite{jin2024sign} briefly introduces the strongest attack, noting that if the average probability of adversary workers sending the gradient's opposite sign is 1, then a limit can be derived on the maximum number of adversary workers required to maintain convergence. Average capacity of attackers is defined as a function of the average probability of honest workers recognizing the correct sign (see Definition 2, Section V in \cite{jin2024sign}).
This results in a probabilistic measure of the strength of the adversaries, which is maximized (on coordinate $i$) when \( q_{j,i}^{(t)} = 1 \), $\forall j \in \calB$.

Within their threat model the authors demonstrate the convergence of the signSGD algorithm through two steps. First, they limit the probability of incorrect aggregation with only \( M \) honest workers. Second, they introduce the presence of \( \calB \) attackers, each with a probability of choosing the gradient's opposite sign. This provides a limit on the number of adversarial workers, depending on the honest workers’ probability of selecting the correct sign. Step 1 is proved in Theorem 1, while Step 2 is proved in Theorem 4 of~\cite{jin2024sign}.

In the following section, we analyze our Byzantine scheme compared to the one presented in~\cite{jin2024sign}.

\subsection{Incorrect Aggregation Considering Both \(\calB\) Attackers and \(M\) Honest Workers}
By defining \(\calB\) adversaries as in \cite{jin2024sign}, the non-convergence bound is established using a positive constant \(c\) that bounds the probability of incorrect aggregation when more than half of the workers send the wrong sign. This considers both honest workers  and attackers. The strongest attack leads to the same bound regarding adversarial workers: \(\alpha Q < (2p - 1) M\), where \( p \) is the probability of a worker correctly identifying the sign. Suppose, without loss of generality, that $p_{i,j}^{(t)} = p$, $\forall i = 1,...,d$, $\forall t = 1,...,T$, $\forall j \in M$. Let \( K \in \calB\) denote the subset of attackers such that, given \( Q \) total workers, \( K = \alpha  Q \) and \( M = (1 - \alpha) Q \).

\paragraph{Proof Sketch.}
Let $\alpha \leq 1 - \frac{1}{2p}$, thus $ K < (1 - \frac{1}{2p})Q$ as in our findings, a straightforward computation leads to $K < (2p - 1) M$.

Let $K$ denote the bound on the $i$-th dimension, as consequence of the constant probability $p$ across the $d$-dimensions, the latter coincide with the bound found in the strongest attack claimed in Theorem 4 of \cite{jin2024sign}, by considering $\bar p = 1 -p$ in accordance to the notation used by the authors.\\
Lastly, a direct comparison on the probability bounds is challenging because a closed form to compute the constant \( c \) is not provided. Instead, leveraging the noise distribution assumption, we present a closed form based on aggregation errors.

In conclusion, while both studies examine the incorrect aggregation and convergence bounds of signSGD against adversaries, our work extends the framework with a comprehensive definition of adversary omniscience and strongest damage, considers the collusion of attackers and ultimately, provides explicit probability bounds without resorting to an unknown constant.

\section{Experimental Results}
\label{sec:experiments}
Similarly to what was done in \cite{bernstein2018signsgd}, we first test out findings on a toy example, before testing them on the standard MNIST dataset. The setting is a 1000-dimensional quadratic with $\mathcal{N}(0, 1)$ noise added to each gradient component. The results are shown in \ref{fig:toy_example_batch_1_500} and demonstrate the influence of the batch size on the convergence behavior in presence of varying proportion of adversaries. 

\begin{figure}[ht!]
    \centering
    \begin{subfigure}[b]{0.49\textwidth}
        \includegraphics[width=\textwidth]
        {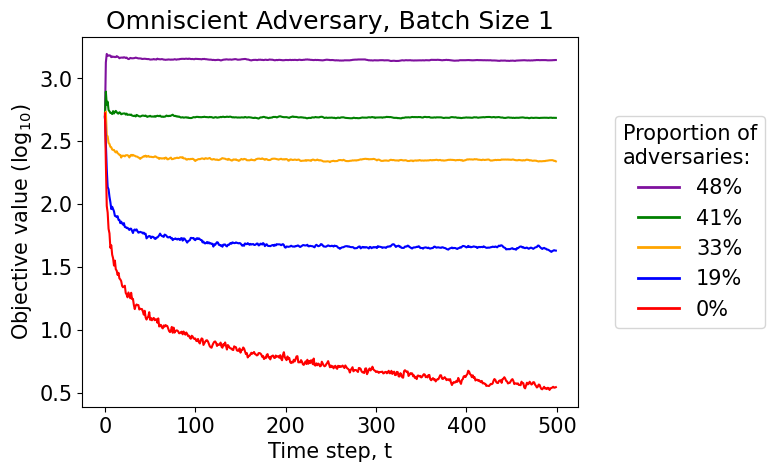}
    \end{subfigure}
    \hfill
    \begin{subfigure}[b]{0.49\textwidth}
        \includegraphics[width=\textwidth]
        {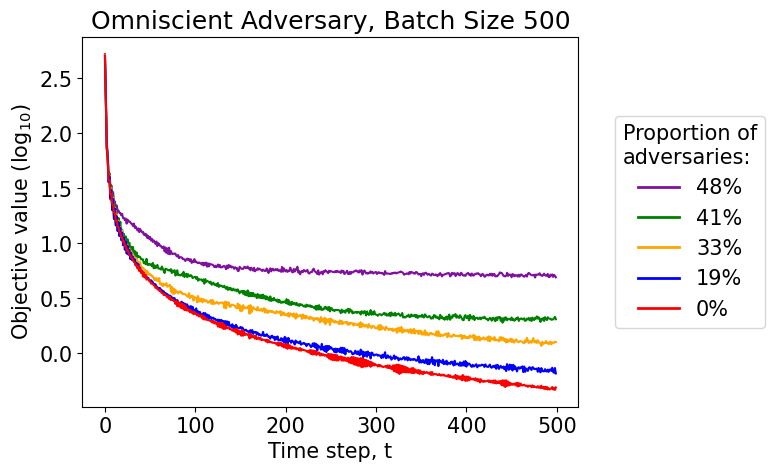} 
    \end{subfigure}
    \caption{Influence of the batch size on the convergence of signSGD with majority vote in the presence of omniscient adversaries, for the toy example with 27 workers and varying numbers of adversaries.}
    \label{fig:toy_example_batch_1_500}
\end{figure}

The advantage of the toy example is that the assumption of Gaussian noise and the assumption of the adversaries knowing the exact gradient are met perfectly. 
The toy example notebook can be found using this \href{https://colab.research.google.com/drive/1vVhphWR9kPUxGH--dJAMoqIxedUTQvDh#scrollTo=cpxNFAP-MdQ0}{\textbf{link}}.

As a more relevant experiment, the MNIST dataset is used, which consists of 60000 training samples and 10000 testing samples. Figures \ref{fig:mnist64}, \ref{fig:mnist256} and \ref{fig:mnist512} show on the left the training loss and on the right the test accuracy over all iterations.

\begin{figure}[h!]
    \centering
    \begin{subfigure}[b]{0.49\textwidth}
        \includegraphics[width=\textwidth]
        {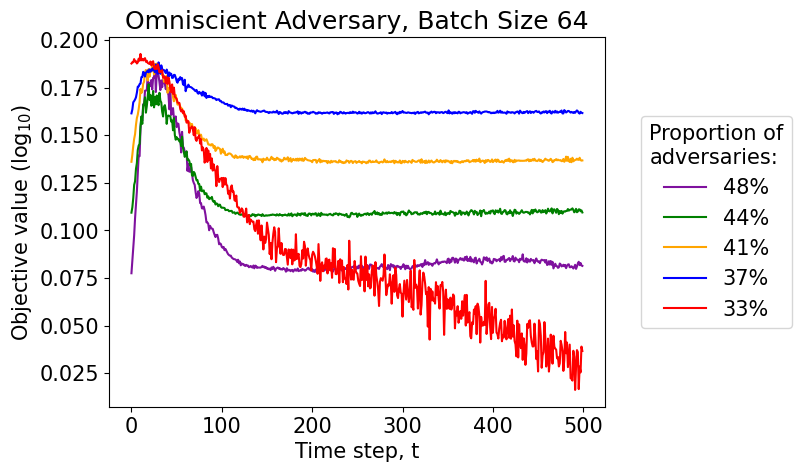}
    \end{subfigure}
    \hfill
    \begin{subfigure}[b]{0.49\textwidth}
        \includegraphics[width=\textwidth]
        {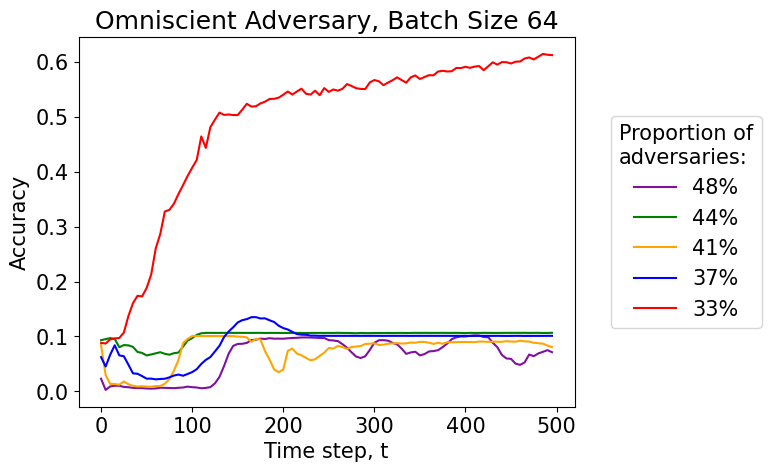}
    \end{subfigure}
    \caption{Training loss and test accuracy for batch size 64 and 500 iterations shows no convergence for more than 33\% adversaries.}
    \label{fig:mnist64}

    \centering
    \begin{subfigure}[b]{0.49\textwidth}
        \includegraphics[width=\textwidth]
        {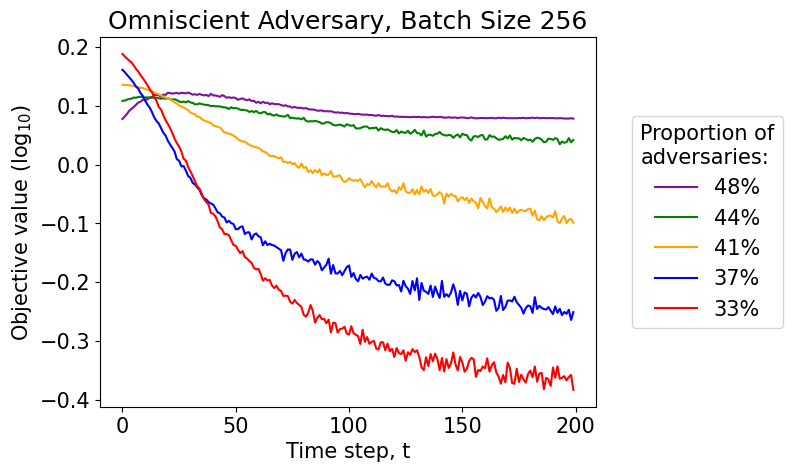}
    \end{subfigure}
    \hfill
    \begin{subfigure}[b]{0.49\textwidth}
        \includegraphics[width=\textwidth]
        {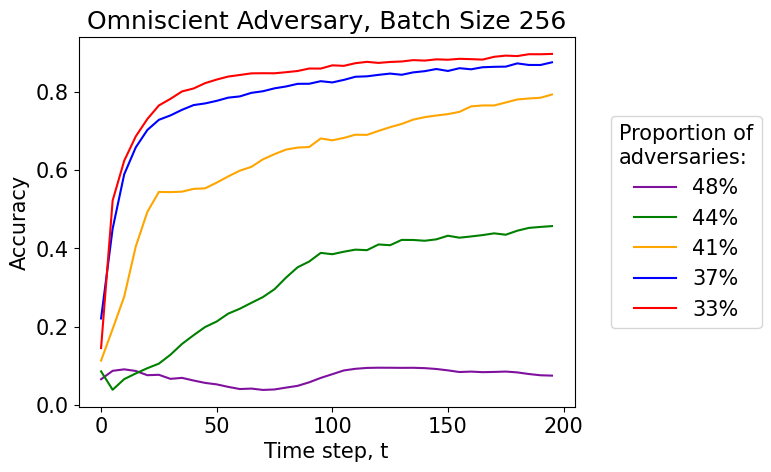}
    \end{subfigure}
    \caption{Training loss and test accuracy for batch size 256 and 200 iterations shows no convergence for 48\% adversaries.}
    \label{fig:mnist256}

    \centering
    \begin{subfigure}[b]{0.49\textwidth}
        \includegraphics[width=\textwidth]
        {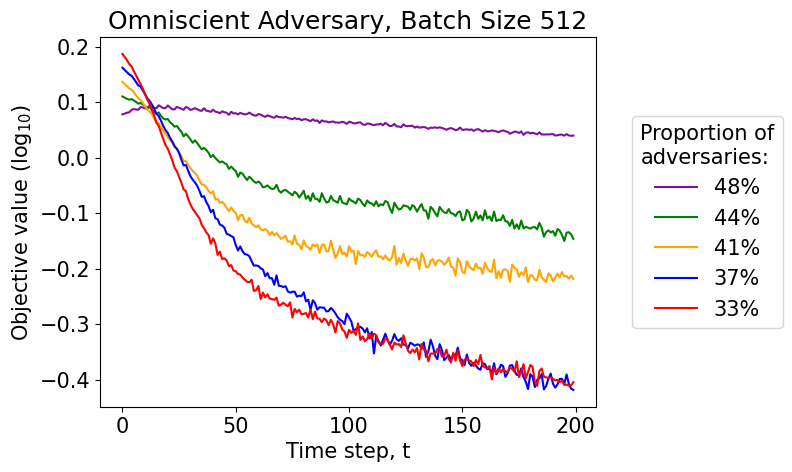} 
    \end{subfigure}
    \hfill
    \begin{subfigure}[b]{0.49\textwidth}
        \includegraphics[width=\textwidth]
        {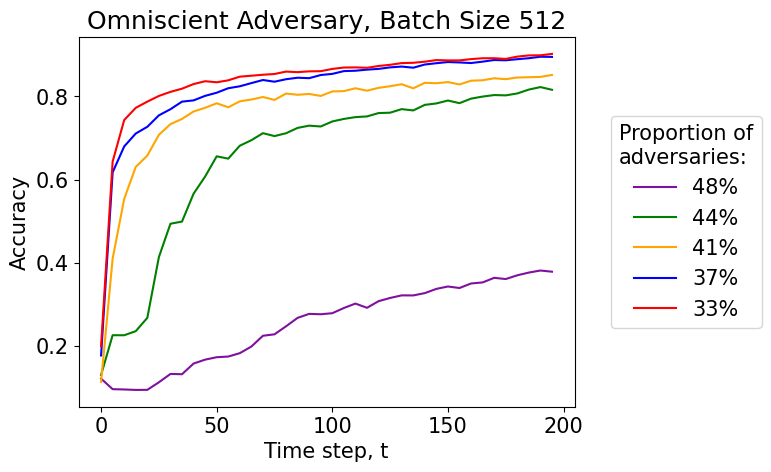} 
    \end{subfigure}
    \caption{Training loss and test accuracy for batch size 512 and 200 iterations shows convergence for up to 48\% adversaries.}
    \label{fig:mnist512}
\end{figure}

All of the $\calB$ omnisicent adversaries compute a gradient on their minibatch.
Those $\calB$ gradients are aggregated on an adversary server with a majority vote to estimate the true gradient and the outcome $\gradest$ is communicated to all adversaries.
Every adversary sends then the opposite of $\gradest$ to the main server where the gradients of all the honest and dishonest workers are aggregated by a majority vote. 
The CNN is developed in a PyTorch environment and using Cuda. The MNIST notebook can be found using this \href{https://colab.research.google.com/drive/1vQpUZgRLQjmRx4AQnt2UX_iorjHVj9rH#scrollTo=NmR53BEtMdRP}{\textbf{link}}.

The experiments with different batch sizes show that convergence and the learning of useful models highly depends on the batch size, which correlates to the probability $p$ that the honest workers to get the sign of the gradients right. It can be seen that for big enough batch sizes with up to 48\% adversaries useful models can be learnt. All details concerning the experimental settings can be found in the appendix \ref{experimental_setting}.

\section{Conclusion}
\label{sec:conclusion}
In conclusion, this work closes the gap on the Byzantine fault tolerance of signSGD in distributed learning settings. We provide a formal definition of the strongest attack in the context of the signSGD algorithm, thus focusing on the majority vote aggregation rule. We identify the maximum number of Byzantine workers that the algorithm can tolerate and ultimately provide the convergence rate under this scenario. Lastly, we compare our contribution with similar approaches in the literature that address the BFT properties of signSGD.

\paragraph{Limitations.} An important limitation that is beyond the scope of our work is the extent of vulnerability of signSGD when the model is of very high dimensionality. For instance, this latter limitation concerns not only signSGD but the standard formalism of BFT distributed learning~\cite{blanchard2017machine} in its entirety. Recently, lower bounds~\cite{jungle} have been proven, arguing for the practical difficulty~\cite{mhamdi2018hidden, baruch2019little} to secure distributed learning in very high dimension, and learning from different and heterogeneous sources~\cite{karimireddy2021byzantine}. These limitations take form either in the assumptions or in the security guarantees on robust machine learning as the latter are heterogeneity and dimension-dependent~\cite{el2022impossible}.
In this regard, signSGD will at least suffer from the same limitations when used with extremely large models, the precise loss of robustness that is specific to signSGD and not general to any distributed learning scheme is yet to be studied.

\bibliographystyle{alpha}
\bibliography{biblio}

\newpage

\appendix

\section{Proof of Lemma \ref{Lemma 1}: Accuracy of Sign Guessing}
\label{prooflemma}

\begin{proof}

\noindent From Lemma 1 of \cite{bernstein2018signsgd}, we know that:
\begin{equation*}
    1 - p = \Prob[\sign{\gradest_i} \neq \sign{\gradtrue_i}] \leq \begin{cases}
    \frac{2}{9}\frac{1}{S_i^2}, & \text{if } S_i > \frac{2}{\sqrt{3}} \\
    \frac{1}{2} - \frac{S_i}{2\sqrt{3}}, & \text{otherwise.}
    \end{cases}
\end{equation*}

We show that in both cases, this probability can be upper bounded by $\frac{1}{2} - \frac{S_i}{2\sqrt{4 + S_i^2}}$.

\emph{Case 1}: $S_i > \frac{2}{\sqrt{3}}$.
\begin{equation*}
    1 - p \leq \frac{2}{9}\frac{1}{S_i^2} \overset{!}{\leq} \frac{1}{2} - \frac{S_i}{2\sqrt{4 + S_i^2}} \iff \frac{4\sqrt{4 + S_i^2} + 9S_i^3}{18S_i^2\sqrt{4 + S_i^2}} \overset{!}{\leq} \frac{1}{2}
\end{equation*}
Notice that for $S_i > \frac{2}{\sqrt{3}}$, the left hand-side term monotonically decreases until $S_i \approx 1.52$ and monotonically increases after that value. In the first range $S_i \in \left]\frac{2}{\sqrt{3}}, 1.52\right[$, we have a highest value of $\approx 0.42 \leq \frac{1}{2}$, for  $S_i = \frac{2}{\sqrt{3}}$. In the second range, $S_i \in \left[1.52, \infty\right[$, we notice that for $S_i \to \infty$ the left hand side converges to the value $\frac{1}{2}$. Therefore, the bound in this case is satisfied. \\
\emph{Case 2}: $S_i \leq \frac{2}{\sqrt{3}}$.
\begin{equation*}
    1 - p \leq \frac{1}{2} - \frac{S_i}{2\sqrt{3}} \overset{!}{\leq} \frac{1}{2} - \frac{S_i}{2\sqrt{4 + S_i^2}} \iff  \sqrt{3} \overset{!}{\leq} \sqrt{4 + S_i^2}.
\end{equation*}

For the smallest value $S_i  = 0$, the term $\sqrt{4 + S_i^2}$ sufficiently upper bounds the inequality. Since this term monotonically increases for growing values of $S_i$, the inequality is satisfied also in \textit{Case 2} and the bound is proven.

Thus:
\begin{equation*}
    1 - p = \Prob[\sign{\gradest_i} \neq \sign{\gradtrue_i}] \leq \frac{1}{2} - \frac{S_i}{2\sqrt{4 + S_i^2}}.
\end{equation*}

Moreover, we can show that
\begin{equation*}
    \frac{p(1 - p)}{(p - \frac{1}{2})^2} \leq \frac{4}{S_i^2}
\end{equation*}
Indeed, by rearranging the terms we have:
\begin{align*}
& \quad \frac{1}{2} - \frac{S_i}{2 \sqrt{4 + S_i^2}} \geq 1 - p \\
&\iff  \frac{S_i}{2 \sqrt{4 + S_i^2}}\leq p - \frac{1}{2} \\
&\iff  \frac{1}{2}\leq (p - \frac{1}{2})(\frac{\sqrt{4 + S_i^2}}{S_i}) \\
&\iff \frac{1}{4} \leq (p - \frac{1}{2})^2(\frac{4}{S_i^2} + 1) \\
&\iff \frac{1}{4(p - \frac{1}{2})^2} - 1 \leq \frac{4}{S_i^2} \\
&\iff \frac{p(1 - p)}{(p - \frac{1}{2})^2} \leq \frac{4}{S_i^2} 
\end{align*}
This concludes our proof.

\end{proof}

\newpage

\section{Experimental Setting}
\label{experimental_setting}

\begin{table}[ht!]
    \centering
    \caption{Details for the Toy Example}
    \begin{center}
    \begin{small}
    \begin{tabular}{ll}
        \toprule
        Category & Details \\
        \midrule
        Objective Function & $0.5 \times \text{np.dot}(x, x)$ \\
        Parameter Vector Size & 1000 \\
        Noise Scale ($\sigma$) & 1.0 \\
        Device & CPU \\
        Total Workers & 27 \\
        Batch Size & 1, t (iteration counter) or 500  \\
        Iterations (T) & 500 \\
        Initial Learning Rate & 1.0 \\
        Learning Rate Schedule & $\text{initial\_lr} / \sqrt{t+1}$ \\
        Repeats & 5 (for worker tests), 1 or 3 (for adversary tests) \\
        Optimizer & signSGD with majority vote \\
        Adversaries & 
        \begin{tabular}{l}
            - Blind: Flipping based on individual gradient estimates \\
            - Omniscient: Flipping based on knowledge of true gradients
        \end{tabular} \\
        Evaluation Metrics & Objective value every iteration \\
        Machine & Windows, Intel Core i7 \\
        Computational Time & 
        \begin{tabular}{l}
            - approx. 3 sec for 1 repeat, batch size 1 and 500 iterations \\
            - approx. 3 min 19 sec for 1 repeat, batch size 500 and 500 iterations 
        \end{tabular} \\
        \bottomrule
    \end{tabular}
    \end{small}
    \end{center}
\end{table}

\begin{table}[ht!]
    \centering
    \caption{Details for the MNIST Experiment}
    \begin{center}
    \begin{small}
    \begin{tabular}{ll}
        \toprule
        Category & Details \\
        \midrule
        Dataset & MNIST \\
        Training Set & Split into 27 chunks for workers \\
        Test Set & Standard MNIST test set \\
        Transformations & 
        \begin{tabular}{l}
            - Center crop to 26x26, resize to 28x28 \\
            - Random brightness, contrast, saturation, and hue adjustments \\
            - Random rotation (±10 degrees) and affine transformation (±5 degrees) \\
            - Normalize with mean 0.1307 and std 0.3081
        \end{tabular} \\
        Model & Custom CNN with 2 conv layers and 2 linear layers \\
        Layers & 
        \begin{tabular}{l}
            - conv1: 1 input, 32 output, kernel size 3 \\
            - conv2: 32 input, 64 output, kernel size 3 \\
            - fc1: 9216 input, 128 output \\
            - fc2: 128 input, 10 output \\
            - Dropout: 0.25 and 0.5
        \end{tabular} \\
        Activation Functions & ReLU, log-softmax in output layer \\
        Device & GPU if available, otherwise CPU \\
        Total Workers & 27 \\
        Batch Size & 64, 256 or 512 (specified in title of plots) \\
        Iterations (T) & 200 (for batch size 256 and 512), 500 (for batch size 64) \\
        Learning Rate ($\eta$) & 0.0001 \\
        Weight Decay ($\lambda$) & 0.05 \\
        Momentum ($\beta$) & 0.9 \\
        Repeats & 1 \\
        Optimizer & signSGD with majority vote \\
        Loss Function & Negative Log-Likelihood Loss (F.nll\_loss) \\
        Adversaries & Simulated by flipping the sign of the aggregated adversary gradient \\
        Evaluation Metrics & Training loss and test accuracy every 5 iterations \\
        Computational Time & 
        \begin{tabular}{l}
            - approx. 50 min for 1 repeat, batch size 64 and 500 iterations \\
            - approx. 65 min for 1 repeat, batch size 512 and 200 iterations 
        \end{tabular} \\
        \bottomrule
    \end{tabular}
    \end{small}
    \end{center}
\end{table}

\end{document}